\documentclass[10pt]{article} % For LaTeX2e
\usepackage[preprint]{tmlr}
\usepackage{amsmath,amsfonts,bm}

\newcommand{\data}{x}
\newcommand{\dataSpace}{\mathcal X}

\newcommand{\veps}{\varepsilon}

\newcommand{\ee}{{\rm e}\hspace{1pt}}

\def\eqref#1{equation~\ref{#1}}
\def\Eqref#1{Equation~\ref{#1}}
\def\1{\bm{1}}

\DeclareMathAlphabet{\mathsfit}{\encodingdefault}{\sfdefault}{m}{sl}
\SetMathAlphabet{\mathsfit}{bold}{\encodingdefault}{\sfdefault}{bx}{n}

\def\gA{{\mathcal{A}}}

\def\gN{{\mathcal{N}}}
\def\gO{{\mathcal{O}}}

\def\sP{{\mathbb{P}}}

\usepackage{algorithm}
\usepackage{algpseudocode}

\usepackage{hyperref}
\usepackage{url}

\usepackage{amsmath}
\usepackage{amssymb} % this raises an error
\usepackage{mathtools}
\usepackage{amsthm}
\usepackage{amsfonts}       % blackboard math symbols
\usepackage{nicefrac}       % compact symbols for 1/2, etc.

\usepackage{microtype}
\usepackage{graphicx} 
\usepackage{subcaption}
\usepackage{comment}

\theoremstyle{plain}
\newtheorem{thm}{Theorem}[section]

\newtheorem{lem}[thm]{Lemma}

\theoremstyle{definition}
\newtheorem{defn}[thm]{Definition}
\newtheorem{assumption}[thm]{Assumption}
\theoremstyle{remark}

\newtheorem{example}[thm]{\textit{Example}}

\graphicspath{{figs/}}

\title{On Using Secure Aggregation in Differentially Private \\Federated Learning with Multiple Local Steps}

\author{\name Mikko A. Heikkil\"a \email mikko.a.heikkila@helsinki.fi \\
      \addr Department of Computer Science\\
      University of Helsinki
      \thanks{Work done partly at Telefónica Research, Barcelona.}
      }

\def\month{03}  % Insert correct month for camera-ready version
\def\year{2025} % Insert correct year for camera-ready version
\def\openreview{\url{https://openreview.net/forum?id=uxyWlXPuIg}} % Insert correct link to OpenReview for camera-ready version

\begin{document}

\maketitle

\begin{abstract}

Federated learning is a distributed learning setting where the main aim is to train machine learning models without having to share raw data but only what is required for learning. To guarantee training data privacy and high-utility models, differential privacy and secure aggregation techniques are often combined with federated learning. However, with fine-grained protection granularities, e.g., with the common sample-level protection, the currently existing techniques generally require the parties to communicate for each local optimization step, if they want to fully benefit from the secure aggregation in terms of the resulting formal privacy guarantees. In this paper, we show how a simple new analysis allows the parties to perform multiple local optimization steps while still benefiting from using secure aggregation. We show that our analysis enables higher utility models with guaranteed privacy protection under limited number of communication rounds.
\end{abstract}

%%%%%%%%%%%%%%%%%%%%%%%%%%%%%%%%%%%%%%%%%%%%

\section{Introduction}
\label{sec:intro}

Federated learning (FL; \citealt{McMahan_2016, Kairouz_et_al_2019}) is a common distributed learning setting, where a central server and several clients holding their own local data sets collaborate to train a single global model. The main feature in FL is that the clients do not directly communicate data, but only what is required for learning, e.g., gradients or updated model parameters (pseudo-gradients).

While FL satisfies the data minimization principle, i.e., only what is actually needed is communicated while the actual raw data never leaves the client, it does not protect against privacy attacks such as membership inference \citep{Shokri2016-zl} or reconstruction \citep{Fredrikson2014-go,Yeom2017-nv}.
Instead, training data privacy is commonly ensured by combining differential privacy (DP; \citealt{dwork_et_al_2006}), a formal privacy definition, and secure multiparty computation (MPC; \citealt{Yao_1982}) with FL (see, e.g., \citealt{Kairouz_et_al_2019}).

DP is essentially a robustness guarantee for stochastic algorithms, which guarantees that small perturbations to the inputs have small effects on the algorithms' output probabilities. What constitutes a small perturbation depends on the chosen protection granularity: the same basic DP definition can be used for ensuring privacy on anything from single sample to entire data set level. 
In turn, MPC protocols can be used to limit the amount of information an adversary has about computations. In FL, secure aggregation (SecAgg) protocols, a specialised form of secure computation that requires significantly less resources than general MPC, are commonly used for communicating model updates from the clients to the server. This approach enables provably better distributed DP (DDP) guarantees, that rely jointly on a group of clients, than is possible to achieve by any single client in isolation.

Under the general FL setup, two main alternatives are commonly considered: cross-device FL and cross-silo FL \citep{Kairouz_et_al_2019}. In cross-device FL, each client is assumed to have a small local data set, while the total number of clients is large, e.g., thousands or millions. In the cross-silo case, the total number of clients is small, for example, a dozen, but each client is assumed to have a larger local data set. 
In this paper, our running example is standard cross-silo differentially private FL (DPFL) where the clients communicate all updates to the server using SecAgg.%
\footnote{Instead of considering any specific SecAgg implementation, in this work we mostly assume an idealised trusted aggregator. We discuss practical implementations in Appendix~\ref{sec:SecAgg_in_practice}.} 
In this setting, the most useful DP protection granularity is typically something strictly more fine-grained than client-level: when clients are, e.g., different hospitals or banks, there are typically several individuals in a single clients' local data set and the protection granularity needs to match the use case.

While client-level granularity in DPFL is, at least in principle, straightforward to combine with SecAgg, more fine-grained granularities such as sample-level DP can present problems: using existing techniques one has to choose between i) having DDP guarantees with less noise due to SecAgg but with all clients using only a single local optimization step per FL round, and ii) having more noisy local DP (LDP) guarantees that do not formally benefit from SecAgg while allowing the clients to do more local optimization steps per FL round. 
Both of these options have significant drawbacks: the amount of server-client communications is typically one of the first bottlenecks that limit model training in FL, while LDP guarantees regularly require noise levels that heavily affect the resulting model utility. In this paper we show that this trade-off is not unavoidable but can be largely remedied by a simple new analysis of the problem.

\paragraph{Our Contribution}
\begin{itemize}
    \item We present a novel and simple theoretical privacy analysis showing when we can increase the number of local optimization steps in FL using fine-grained DP granularity, 
    while still benefiting from DDP guarantees using a trusted aggregator.
    \item We demonstrate empirically that the proposed approach can lead to large utility benefits (in terms of prediction accuracy and loss) 
    without requiring any changes to the underlying algorithms under both iid and heterogeneous client data splits. In our experiments with limited number of global communication rounds, using a convolutional neural network with Fashion MNIST data set as well as linear models on CIFAR-10 (using pre-trained model as feature extractor) and ACS Income data sets, we improve prediction accuracy roughly by $16\%$ points, $4\%$ points, and $4\%$ points, respectively (corresponding to $23\%$, $5\%$ and $5\%$ improvements).
    \item Our results point to a mismatch between the current theoretical understanding of vanilla DPFL (standard DPFL with FedAvg) and practical results.
\end{itemize}

%%%%%%%%%%%%%%%%%%%%%%%%%%%%%%%%%%%%%%%%
%%%%%%%%%%%%%%%%%%%%%%%%%%%%%%%%%%%%%%%%

\section{Related Work}
\label{sec:related_work}

There is a significant amount of existing work focusing on the general problem of combining DP with FL, although the focus has mostly been on the cross-device FL setting with user- or client-level DP. To the best of our knowledge, while the combination of DPFL with SecAgg is certainly not novel (see, e.g., \citealt{Truex_et_al_2019, Kairouz_et_al_2019, heikkila_2020,  Stevens_Skalka_Vincent_Ring_Clark_Near_2021, Yang_2023}), there is very limited amount of work specifically on the privacy analysis when the clients do multiple local optimization steps with fine-grained DP and communicate the results via SecAgg. 
As far as we know, the only directly relevant analysis is by \cite{Noble_Bellet_Dieuleveut_2023}, who consider the Gaussian mechanism for sample-level DP in the cross-device setting without secure aggregation. Our main theorem (Theorem~\ref{thm:dpsgd_with_local_steps}) can be seen as extending their result on the Gaussian mechanism to other DP mechanisms that are more suitable for combining with actual secure aggregation protocols.
\footnote{Note that \citep[Algorithm~4]{Truex_et_al_2019} also state a similar specialised version for the Gaussian mechanism, i.e., they use several local optimization steps with sample-level DP and SecAgg in FL, while scaling the Gaussian noise for DDP considering all the clients. However, as also noted by \cite{Malekzadeh_2021}, 
the approach of \cite{Truex_et_al_2019} would require a separate proof of privacy beyond what is actually provided in the paper.}

Considering other existing work in more detail, we can distinguish some main lines of closely-related research. 
Besides the DDP approach we focus on, i.e., the combination of LDP noise with secure aggregation, the use of LDP mechanism by each client independently provides robust privacy with few assumptions, but generally results in worse utility (see, e.g., \citealt{Kasiviswanathan_Lee_Nissim_Raskhodnikova_Smith_2008,DBLP:conf/eurosys/Truex0CGW20}). Another recently considered approach assumes a trusted server that is tasked with the DP noise addition, and does privacy accounting via the matrix mechanism \citep{DBLP:journals/corr/abs-2411-18752}. This can be shown to improve the privacy-utility trade-off compared to standard DP-SGD under some assumptions, but is not directly compatible with the secure aggregation approach for DDP we focus on in this work.

Looking at somewhat more closely-related papers in terms of the general setting, there is plenty of work proposing novel learning methods for FL, assuming sample-level DP and LDP noise with SecAgg for improved DDP guarantees. While the existing work only uses a single local optimization step (see, e.g., \citealt{heikkila_2020, Malekzadeh_2021, Stevens_Skalka_Vincent_Ring_Clark_Near_2021, Yang_2023}), our analysis can be leveraged in this setting to enable running multiple local steps generally for many such methods without requiring any other changes to the algorithms.

Another clear line of work has focused on introducing novel discrete DP mechanisms that can be used with additively homomorphic encryption techniques, which typically operate on the group of integers with modulo additions. 
\cite{Agarwal_Suresh_Yu_Kumar_Mcmahan_2018} proposed 
a binomial mechanism that provides DP using 
discrete binomial noise. Improving on the binomial mechanism, \cite{Canonne_Kamath_Steinke_2020} proposed a discrete Gaussian mechanism, while \cite{Agarwal_Kairouz_Liu_2021} introduced a Skellam mechanism and \cite{Chen_et_al_2022} a Poisson-binomial mechanism, both of which improve on the discrete Gaussian, e.g., by being infinitely divisible distributions: the sum of Skellam/Poisson-binomial distributed random variables is another Skellam/Poisson-binomial random variable. Our work is not focused on introducing new DP mechanisms, but our analysis allows for using many different DP noise mechanism. In particular, our analysis allows for using LDP noise with SecAgg including when using infinitely divisible DP mechanisms, such as the Skellam mechanism, with pseudo-gradients and fine-grained DP protection level.

While our main focus is on privacy accounting with SecAgg under limited communication budget, there has also been considerable effort by the community to reduce the amount of required communication further by applying quantization to the gradients \citep{
Agarwal_Suresh_Yu_Kumar_Mcmahan_2018, Kairouz_Liu_Steinke_2021, Agarwal_Kairouz_Liu_2021, Chen_et_al_2022, Jin_Huang_He_Dai_Wu_2021, Chaudhuri_Guo_Rabbat_2022, Guo_Chaudhuri_Stock_Rabbat_2023} 
or by compressing the updates sent by the clients 
\citep{Triastcyn_Reisser_Louizos_2021, Chen_Choquette-Choo_Kairouz_Suresh_2022}. 
In principle, any such technique for compressing the model updates compatible with SecAgg can also be directly combined with our privacy analysis. In contrast, benefiting from gradient quantization is not entirely straightforward as in our case the model updates are pseudo-gradients and not gradients. 
We leave a detailed consideration and comparison of the possible methods for reducing the required communication budget beyond what is possible by pushing more optimization steps to the clients for future work.

In summary, while many of the contributions cited above, e.g., novel DP mechanisms, are not limited to cross-device FL, all the experiments and use cases mentioned in the cited papers that are compatible with SecAgg and use multiple local steps only consider DDP with \emph{user- or client-level DP} in cross-device FL. In contrast, 
we focus on more fine-grained DP granularities, namely on \emph{sample-level DP}. As we discuss in Section~\ref{sec:background}, combining sample-level DP with multiple local steps and LDP noise with SecAgg for improved utility requires a novel privacy analysis. The main aim of this paper is to provide such an analysis.

While the currently existing theoretical convergence bounds for vanilla DPFL with FedAvg do not show any benefit from increasing the number of local steps in DPFL (see \citealt[Theorem 3.2]{Malekmohammadi_Yu_Cao_2024}), we empirically demonstrate the utility of our analysis in Section~\ref{sec:experiments} after stating the results in Section~\ref{sec:local_steps_with_secsum}. Our results clearly highlight the need for improving the theoretical analysis of standard DPFL over what is shown by \cite{Malekmohammadi_Yu_Cao_2024} to understand when increasing the number of local steps is useful (compare this disagreement of empirical results and theory to the discussion by \citealt{pmlr-v162-mishchenko22b} on the provable usefulness of local steps in vanilla non-DP FL).

%%%%%%%%%%%%%%%%%%%%%%%%%%%%%%%%%%%%%%%%
%%%%%%%%%%%%%%%%%%%%%%%%%%%%%%%%%%%%%%%%
\section{Background}
\label{sec:background}

Federated learning (FL, \citealt{McMahan_2016, Kairouz_et_al_2019}) is a collaborative learning setting, where the participants include a central server and clients holding some data. On each FL round, the server chooses a group of clients for an update and sends them the current model parameters. The chosen clients update their local model parameters by taking some amount of optimization steps using only their own local data, and then send an update back to the server. The server then aggregates the client-specific contributions to update the global model. 
We use the standard federated averaging update rule: assuming w.l.o.g. that clients $i=1,\dots,N$ have been selected at FL round $t$, and that client $i$ sends an update $\Delta^{(t)}_i$ (pseudo-gradient), the updated global model $\theta_t$ is given by 
\begin{equation}
    \theta_t = \theta_{t-1} + \frac{1}{N} \sum_{i=1}^N \Delta_i^{(t)} .
\end{equation}

\subsection{Differential Privacy}

We want to guarantee privacy of the trained model w.r.t. the training data, for which we use differential privacy (DP). Writing the space of possible data sets as $\dataSpace^*:= \cup_{n \in \mathbb N} \dataSpace^n$, we have the following:
\begin{defn}
    \label{def:dp}
    \citep{dwork_et_al_2006,dwork2006our}
	Let $\varepsilon > 0$ and $\delta \in [0,1]$. 
	A randomised algorithm $\gA: \, \dataSpace^* \rightarrow \mathcal{O}$ is  $(\veps, \delta)$-DP 
	if for every $\data, \data' \in \dataSpace^*: \data \simeq \data'$, and every measurable set $E \subset \mathcal{O}$, 
	\begin{equation*}
		\begin{aligned}
			&\sP ( \gA (\data) \in E ) \leq \ee^\varepsilon \sP (\gA (\data') \in E ) + \delta,
		\end{aligned}
	\end{equation*}
    where $\simeq$ is a neighbourhood relation. 
	$\gA$ is tightly $(\veps,\delta)$-DP, 
	if there does not exist $\delta' < \delta$
	such that $\gA$ is $(\veps,\delta')$-DP. 
	When $\delta=0$, we write $\veps$-DP and call it \emph{pure DP}. The more general case $(\veps,\delta)$-DP is called \emph{approximate DP} (ADP).
\end{defn}

Definition~\ref{def:dp} can be equivalently stated as a bound on the so-called hockey-stick divergence:
\begin{defn}
    \label{def:hockey-stick-div}
    Let $\alpha > 0$. The hockey-stick divergence between distributions $P,Q$ is given by 
    \begin{equation}
        H_{\alpha} (P \| Q) 
        := \mathbb E_{t \sim Q} \left( \left[\frac{dP}{dQ}(t) - \alpha \right]_+ \right) 
        = \mathbb E_{t \sim Q} \left( \left[\frac{p(t)}{q(t)} - \alpha \right]_+ \right) ,
    \end{equation}
    where $[a]_+ := \max(a,0)$, $\frac{dP}{dQ}$ is the Radon-Nikodym derivative, and $p,q$ are the densities of $P,Q$, respectively. In the rest of this paper, we assume that all the relevant densities exists.
\end{defn}

It has been shown that a randomised algorithm $\gA$ is $(\varepsilon,\delta)$-DP iff $\sup_{\data \simeq \data'} H_{\ee^{\varepsilon}} (\gA(\data) \| \gA (\data')) \leq \delta$ \citep{Barthe_2013,Barthe_Olmedo_2013}.

Our main results do not depend on the exact neighbourhood definition, but in all the experiments we use the add/remove relation or unbounded DP, that is, $\data, \data' \in \dataSpace^*$ are neighbours, if $\data$ can be transformed into $\data'$ by adding or removing a single protected unit from $\dataSpace$. For the protection granularity, although this is not a strict limitation, we focus on the common sample-level DP, i.e., a single protected unit corresponds to a single sample of data. As noted earlier, our analysis could be advantageous for anything more fine-grained than client-level DP, e.g., element-level DP \citep{Asi_Duchi_Javidbakht_2019} or even for individual-level when there are more than one individual in the clients' local data.

We also make use of dominating pairs:
\begin{defn}
    \label{def:dominating_pairs}
    \citep{Zhu2021-kb}
	A pair of distributions $(P,Q)$ is a dominating pair for a stochastic algorithm $\gA$, if for all $\alpha \geq 0$ 
     \begin{equation}
     \sup_{\data,\data' \in \dataSpace^*: \data \simeq \data'} H_{\alpha} \left(\gA (\data) \| \gA (\data') \right) \leq H_{\alpha} \left( P \| Q \right),
     \end{equation}
     where $H_{\alpha}$ is the hockey-stick divergence (Definition~\ref{def:hockey-stick-div}).
\end{defn}

\subsection{Problem with Local Steps in DPFL with SecAgg}

While DP offers strict privacy protection, it comes at the cost of reduced model utility. This is especially true in the local DP (LDP) setting, where each client protects its own data independently of any other party \citep{Kasiviswanathan_Lee_Nissim_Raskhodnikova_Smith_2008}. One well-known technique to improve model utility in DPFL has been to utilise secure aggregation (SecAgg) to turn LDP guarantees into distributed DP (DDP) guarantees that depend on multiple clients (see, e.g., \citealt{Kairouz_et_al_2019}).%
\footnote{Here, we distinguish between local DP (LDP, each client guarantees DP for it's own contribution locally and in isolation), distributed DP (DDP, a group of clients jointly contribute to the resulting DP guarantees against outside adversaries without a trusted central party) and centralised DP (CDP, a trusted party guarantees DP for everyone).} 
In effect, this means that when the clients are able to collaborate, they can jointly scale the locally used noise levels to reach a target DDP guarantee resulting in improved utility compared to the LDP guarantees. Even if the clients do not explicitly collaborate on calibrating the noise levels, the DDP guarantees resulting from using SecAgg will nevertheless be stronger compared to the LDP guarantees reached by each client individually.

However, naively combining fine-grained DP protection with SecAgg for DDP runs into problems, as we demonstrate in the rest of this section. 
For convenience, we give the pseudo-code for running standard DPFL with sample- and user-level DP protection in Appendix~\ref{sec:standard_algos}.

Starting with the unproblematic case of client-level DP, writing $TA$ for an ideal trusted aggregator and using the well-known Gaussian mechanism \citep{dwork2006our} for simplicity, one can get DDP guarantees for any number of local optimization steps with the following update:
\begin{equation}
    \label{eq:param-clip-and-noise}
    \theta_t = \theta_{t-1} + \frac{1}{N} TA \left( \sum_{i=1}^N clip_C(\Delta^{(t)}_i) + \xi^{(t)}_i \right) ,
\end{equation}
where the sum inside $TA$ is done by a trusted aggregator, $clip_C$ ensures that each client-specific update has $\ell_2$-norm bounded by the constant $C$, and $\xi^{(t)}_i$ is Gaussian noise s.t. $\sum_{i=1}^N \xi^{(t)}_i$ gives the DDP protection level we are aiming for. 
As the clipping and noise are applied directly to the updated weights after the local optimization has finished, the privacy protection is not affected by the number of local optimization steps client $i$ is using to arrive at $\Delta_i^{(t)}$ before applying DP.

There is also a simple approach that works with more fine-grained granularities, when the clients use a single local optimization step with common learning rate $\gamma$ and, for example, standard DP stochastic gradient descent (DP-SGD, \citealt{Song_Chaudhuri_Sarwate_2013}) again utilising Gaussian noise: we can take $\Delta_i^{(t)} = -\gamma (g_i^{(t)} + \xi_i^{(t)})$, where $g_i^{(t)}$ is a sum of clipped per-unit gradients (e.g. per-sample for sample-level DP) from client $i$, to have the update 
\begin{equation}
    \label{eq:single-step-gaussian-example}
    \theta_t = \theta_{t-1} - \frac{1}{N} TA \left( \sum_{i=1}^N \gamma ( g_i^{(t)} + \xi^{(t)}_i) \right) .
\end{equation}
Looking at the sum in \Eqref{eq:single-step-gaussian-example}, since each per-unit gradient has a common bounded norm and Gaussian noise is infinitely divisible, i.e., the summed-up noise is another Gaussian, we can calculate the resulting privacy with standard techniques (see, e.g., \citealt{Mironov_Talwar_Zhang_2019, Koskela2020-ol, Zhu2021-kb}). 
Now, if one tries to use the same reasoning with sample-level DP using $S > 1$ local optimization steps, the problem is that the sensitivity of the per-sample clipped gradients when summed over the local steps increases with $S$: assuming $\| g_{i,s} \|_2 \leq C, s=1,\dots,S$ implies that $\| \sum_{s=1}^S g_{i,s} \|_2 \leq SC $ (triangle-inequality).%
\footnote{Note that, as already mentioned in Section~\ref{sec:related_work}, for the special case of Gaussian noise with a trusted server, the analysis of \cite{Noble_Bellet_Dieuleveut_2023} shows that we can indeed avoid the problem and get valid DDP guarantees that improve on the LDP case.}

In other words, trying to scale the noise over multiple local optimization steps naively ends up scaling the query sensitivity linearly with the total number of steps, while the obvious problem in using only a single step per FL round is that the number of communication rounds is typically one of the main bottlenecks in FL \citep{Kairouz_et_al_2019}. 

%%%%%%%%%%%%%%%%%%%%%%%%%%%%%%%%%%%%%%%%

\subsection{Trust Model}

In this paper, we assume an honest-but-curious (hbc) server and that all the clients are fully honest. The latter assumption can be easily generalised to allow for hbc clients with some weakening to the relevant privacy bounds: with $N$ (non-colluding) hbc clients, since any client could potentially remove its own noise from the aggregated results, the noise from the other $N-1$ clients needs to guarantee the target DP level. In effect, to allow for all hbc clients, we would need to scale up the noise level somewhat (see, e.g., \citealt{Heikkila_2017} for a discussion on the noise scaling and for formal proofs).

In principle, the same technique can also protect against privacy threats in the case of including some fully malicious clients in the protocol (i.e, simply scale the noise so that the hbc clients are enough to guarantee the required DP level). However, in this case the required level of extra noise will increase quickly with the number of malicious clients leading to heavier utility loss. With malicious clients, there would also be no guarantee that the learning algorithm terminates properly.

%%%%%%%%%%%%%%%%%%%%%%%%%%%%%%%%%%%%%%%%
%%%%%%%%%%%%%%%%%%%%%%%%%%%%%%%%%%%%%%%%

\section{Privacy Accounting for Multiple Local Steps Using a Trusted Aggregator}
\label{sec:local_steps_with_secsum}

Consider standard FL setting with $M$ clients and client $i$ holding some local data $\data_i$. On FL round $t$, $N_t$ clients are selected for updating by the server, w.l.o.g. assumed to be clients $i=1,\dots,N_t$. Each selected client $i$ receives the current model parameters $\theta^{(t-1)}$ from the server, then runs $S_t$ local optimization steps using DP-SGD with constant learning rate $\gamma_t$, and finally sends an update to the server via a trusted aggregator $TA$:
\begin{equation}
    \Delta_i^{(t)} 
    = \theta_i^{(t)} - \theta^{(t-1)} 
    = - \sum_{s=1}^{S_t} \gamma_t ( g^{(t)}_{i,s} + \xi^{(t)}_{i,s} ), 
\end{equation}
where we write $g^{(t)}_{i,s}$ for the per-unit clipped gradients of client $i$ at local step $s$, and $\xi^{(t)}_{i,s}$ for the DP noise. 
After receiving all the messages via the trusted aggregator, the server updates the global model using FedAvg:
\begin{equation}
    \theta^{(t)} = \theta^{(t-1)} + \frac{1}{N_t} TA(\sum_{i=1}^{N_t} \Delta_i^{(t)} ).
\end{equation}

In the rest of this section we state our main results: we show that under some assumptions we can account for privacy in FL by looking at the local optimization steps while benefiting from the noise added by all the clients on each step, even if there is no communication between the clients during the local optimization but only a single trusted aggregation at the end of the round to update the global model parameters.

W.l.o.g. from now on we drop the FL round index $t$ and simply write, e.g., $N$ instead of $N_t$ for the number of updating clients. Since the global updates do not access any sensitive data, once we can do privacy accounting for a single FL round, which is the main topic in the rest of this section, generalising to $T$ FL rounds can be done in a straightforward manner (see Appendix~\ref{sec:accounting_details}).

In the following, we assume that all clients have access to an ideal trusted aggregator, and that all sums are calculated by calling the trusted aggregator. We comment on more realistic implementations in Appendix~\ref{sec:SecAgg_in_practice} after stating our main results.

We make the following assumptions throughout this section:
\begin{assumption}
\label{assumption:common_assumptions}
    Let $\data_i \in \dataSpace^*, i=1,\dots,N$. 
    We write $\data = \cup_{i=1}^N \data_i$, and assume that $\data_i \cap \data_j = \varnothing$ for every $i\not=j$, i.e., there are no overlapping samples in different clients' local data sets. 
    We are interested in fixed-length optimization runs of $S$ local steps (common to all clients), which leads to (fixed-length) adaptive sequential composition for privacy accounting (see e.g. \citealt{Rogers_Roth_Ullman_Vadhan_2016,Zhu2021-kb}). 
    We assume all clients use the same learning rate $\gamma$ and norm clipping with constant $C$ when applicable. 
    We also assume that all local DP mechanisms $\gA^{(s)}_{i}, s=1,\dots,S,i=1,\dots,N$ are DP w.r.t. the first argument for any given auxiliary values (which we generally do not write out explicitly).
\end{assumption}
Note that we consider how to loosen many of these assumptions in Appendix~\ref{sec:loosening_assumptions}.

Not all possible DP mechanisms might allow for improved DDP guarantees via simple aggregation. 
For convenience, in Definition~\ref{def:sum-dominating-mechanism} we define a family of suitable mechanisms, which we call sum-dominating:
\begin{defn}[Sum-dominating mechanism]
\label{def:sum-dominating-mechanism}
    Let $\gA, \gA_{i}:\dataSpace^*\rightarrow \mathcal O, i=1,\dots,N$ be randomised algorithms. We call $\gA$ a \emph{sum-dominating} mechanism w.r.t. $\gA_i, i=1,\dots,N$, if 
    \begin{equation}
    \label{eq:sum-dominating-def-eq}
         \sup_{\data\simeq \data'} H_{\alpha} \left( \sum_{i=1}^N \gA_{i}(\data_i) \| \sum_{i=1}^N \gA_{i}(\data'_i) \right) 
        \leq 
        \sup_{\data\simeq \data'} H_{\alpha} \left( \gA(\data) \| \gA(\data') \right) ,
    \end{equation}
    where $H_{\alpha}$ is the hockey-stick divergence, and $\simeq$ is the DP neighbourhood relation. 
    If additionally $\sup_{\data\simeq \data'} H_{\alpha} \left( \gA(\data) \| \gA(\data') \right) < \max_i \sup_{\data\simeq \data'} H_{\alpha} \left( \gA_i(\data_i) \| \gA_i(\data_i') \right)$, then we call $\gA$ a \emph{properly sum-dominating} mechanism.
\end{defn}

Looking at Definition~\ref{def:sum-dominating-mechanism}, in practice the interesting mechanisms for us will be the properly sum-dominating mechanisms, as these have improved privacy bounds compared to any single mechanism $\gA_i$ in isolation. However, from the following results only Theorem~\ref{thm:per-client-gaussian-noise-scaling} is specific to proper sum-dominating mechanisms. 
Considering concrete mechanisms that are properly sum-dominating, simple examples are given by the existing DP mechanisms that use infinitely divisible noise (see Definition~\ref{def:infinite_divisibility}), such as 
Skellam \citep{Valovich_Alda_2017,Agarwal_Kairouz_Liu_2021}, Poisson-binomial \citep{Chen_et_al_2022}, as well as the ubiquitous continuous Gaussian \citep{dwork2006our}:%
\footnote{Discrete Gaussian \citep{Canonne_Kamath_Steinke_2020} is not infinitely divisible, but is close-enough that a properly sum-dominating mechanism can still be found in many practical settings, see \cite{Kairouz_Liu_Steinke_2021}. In such cases, the inequality in Equation~\ref{eq:sum-dominating-def-eq} could always be strict, whereas for the infinitely divisible noise mechanisms mentioned in the text it can be written as an equality. We note that even in this case, however, writing Equation~\ref{eq:sum-dominating-def-eq} with inequality is necessary to avoid nonsensical limitations, such as having a 
DP mechanism that satisfies Definition~\ref{def:sum-dominating-mechanism} with a given $\delta$ while not satisfying it for any $\delta' > \delta$.}\\

\begin{example}[Gaussian mechanism]
\label{example:gaussian_mechanism}
    Assume $\gA_{i}$ is a Gaussian mechanism with noise covariance $C^2 \sigma_i^2 \cdot I_d$ and $f$ has bounded sensitivity $C$. 
    Since the normal distribution is infinitely divisible (Definition~\ref{def:infinite_divisibility}), the combined mechanism $\gA = \sum_{i=1}^N \gA_{i}$ is another Gaussian with sensitivity $C$ and noise covariance $C^2 (\sum_{i=1}^N \sigma^2_i) \cdot I_d$. 
    Finally, due to well-known existing results (see e.g. \citealt{Meiser_Mohammadi_2018, Koskela2020-ol, Zhu2021-kb}), 
    a (tightly) dominating pair of distributions $(P,Q)$ in the sense of Definition~\ref{def:dominating_pairs} for the sum-dominating mechanism $\gA$ is given by a pair of 1d Gaussians with means $\mu_P = 0, \mu_Q =1$, and variances $\sigma^2_P = \sigma^2_Q = \sum_{i=1}^N \sigma^2_i$. As the DP guarantees for the Gaussian improve when increasing the variance \citep{BalleW18}, $\gA$ is a properly sum-dominating mechanism w.r.t. to $\gA_i, i=1,\dots,N$ (Definition~\ref{def:sum-dominating-mechanism}).
\end{example}

Next, we consider composing a sum-dominating mechanisms over $S$ local steps. This allows us to account for the total privacy when doing more than one local optimization step: 
\begin{lem}
\label{lemma:dominating_compositions}
    Assume $\gA^{(s)}$ is a sum-dominating mechanism w.r.t. $\gA_i^{(s)}, i=1,\dots,N$  
    for every $s=1,\dots,S$.
    Then the composition of the sum-dominating mechanisms 
    $(\gA^{(1)}, \dots, \gA^{(S)})$ 
    dominates the composition 
    \begin{equation}
        \left( \sum_{i=1}^N \gA_{i}^{(1)}, \dots, \sum_{i=1}^N \gA_{i}^{(S)} \right) .
    \end{equation}
\end{lem}

\begin{proof}
    For any $s \in \{1,\dots,S\}$, we immediately have 
    \begin{equation}
         \sup_{\data\simeq \data'} H_{\alpha} \left( \sum_{i=1}^N \gA^{(s)}_{i}(\data_i) \| \sum_{i=1}^N \gA^{(s)}_{i}(\data'_i) \right) 
        \leq  \sup_{\data\simeq \data'} H_{\alpha} \left( \gA^{(s)}(\data) \| \gA^{(s)}(\data') \right) 
    \end{equation}
    by definition of $\gA$ (Definition~\ref{def:sum-dominating-mechanism}). 
    The claim therefore follows immediately from \cite[Theorem~10]{Zhu2021-kb}.
\end{proof}

Considering Lemma~\ref{lemma:dominating_compositions}, 
in our case it essentially says that to account for running $S$ local optimization steps, it is enough to find a sum-dominating mechanism for each step separately.

With the next result given in Lemma~\ref{lemma:from_vector_to_exchanged_sums}, we  can connect the previous results with the form of output we get from actually running local optimization in FL:
\begin{lem}
\label{lemma:from_vector_to_exchanged_sums}
    Assume that releasing the vector 
    \begin{align}
        & \left( \sum_{i=1}^N \gA_{i}^{(1)} (\data_i), \dots, \sum_{i=1}^N \gA_{i}^{(S)} (\data_i) \right) 
    \end{align} 
    satisfies $(\varepsilon,\delta)$-DP. 
    Then releasing 
    \begin{equation}
        \sum_{i=1}^N \sum_{s=1}^S \gA_{i}^{(s)} (\data_i) 
    \end{equation}
    also satisfies $(\varepsilon,\delta)$-DP.
\end{lem}

\begin{proof}
    Due to the post-processing immunity of DP (see, e.g., \citealt{Dwork2014-do}), the assumption implies that releasing 
    \begin{align}
        \sum_{s=1}^S \sum_{i=1}^N \gA_{i}^{(s)} (\data_i) 
    \end{align}
    satisfies $(\varepsilon,\delta)$-DP, 
    and by exchanging the order of summation 
    the claim follows. Note that all the mechanisms are assumed to be DP w.r.t. their first argument for any given auxiliary value, which allows us to do the exchange without affecting privacy (in the context of FL, we effectively switch from communicating between each local step to 
    running all local steps and then communicating).
\end{proof}

Taken together, Definition~\ref{def:sum-dominating-mechanism} along with Lemmas~\ref{lemma:dominating_compositions} \& \ref{lemma:from_vector_to_exchanged_sums} 
allow us to compose DP mechanisms for DDP guarantees using LDP noise from multiple clients. In our main result given as Theorem~\ref{thm:dpsgd_with_local_steps}, we show that with (properly) sum-dominating mechanisms each client can run DP-SGD with several local steps and LDP noise while receiving (improved) DDP guarantees when communicating the update via a trusted aggregator.

\begin{thm}
    \label{thm:dpsgd_with_local_steps}
    Assume $N$ clients use local noise mechanisms $\gA_i^{(s)}, i=1,\dots,N$ as in Definition~\ref{def:sum-dominating-mechanism} for each local gradient optimization step $s=1,\dots,S$, 
    and that the final aggregated update $\sum_{i=1}^N \Delta_i$ is released via an ideal trusted aggregator. 
    Then denoting the sum-dominating mechanism for step $s$ by $\gA^{(s)}$, the query release satisfies $(\veps(\delta),\delta)$-DP for any $\delta \in [0,1]$, when $\veps(\delta)$ is given by accounting for releasing the vector 
    $$ \left( \gA^{(1)} (\data), \dots, \gA^{(S)} (\data) \right) ,$$
    where $\data = \cup_{i=1}^N \data_i$.
\end{thm}

\begin{proof}

    For privacy accounting, assuming all sums are done by trusted aggregator $TA$, releasing the aggregated update $TA(\sum_{i=1}^N \Delta_i)$ 
    corresponds to releasing the result 
    $$ -\gamma \sum_{i=1}^N \sum_{s=1}^S \gA^{(s)}_{i} (\data_i; z_{i,s}) , $$
    where each mechanism includes a mapping that maps the local samples to the clipped per-unit gradients as well as the DP noise, and $z_{i,s}$ are auxiliary values (e.g., state after the previous step).%
    \footnote{For example, with standard DP-SGD, sample-level DP and continuous Gaussian noise, 
    $ \sum_{i=1}^N \Delta_i = - \gamma \sum_{i=1}^N \sum_{s=1}^S (g_{i,s} + \xi_{i,s})$, 
    where $g_{i,s}$ are (sums of) clipped per-sample gradients and $\xi_{i,s}$ are the per-step Gaussian noises.} 
    Since all mechanisms are assumed to be DP w.r.t. the first argument for any auxiliary value, the auxiliary values do not affect the DP guarantees, and hence we do not write them explicitly in the following.
    
    From Lemma~\ref{lemma:from_vector_to_exchanged_sums} 
    it follows that valid DP guarantees can be established by accounting for the release of the vector $\left( \sum_{i=1}^N \gA_i^{(1)} (\data_i), \dots, \sum_{i=1}^N \gA_i^{(S)} (\data_i) \right)$. Furthermore, Definition~\ref{def:sum-dominating-mechanism} 
    implies that for any step $s \in {1,\dots,S}$, the sum-dominating mechanism $\gA^{(s)}$ dominates $\sum_{i=1}^N \gA_i^{(s)}$, and therefore by  
    Lemma~\ref{lemma:dominating_compositions} the claim follows.
\end{proof}

Considering properly sum-dominating mechanisms based on infinitely divisible noise, we also have the following theorem on scaling the noise with multiple clients:
\begin{thm}
\label{thm:per-client-gaussian-noise-scaling}
In addition to the assumptions of Theorem~\ref{thm:dpsgd_with_local_steps}, assume that each client uses a common noise level $\sigma_{LDP}$ for some mechanism based on infinitely divisible additive noise for each local step s.t. the DP noise $\xi_{LDP}$ for any given local mechanism $\gA_i^{(s)}, i=1,\dots,N, s=1,\dots,S$, has covariance $\sigma_{LDP}^2 \Sigma_d$ for some $\Sigma_d$. 
Then for guaranteeing given $(\varepsilon(\delta),\delta)$-DDP, the required $\sigma_{LDP}$ scales as $\gO(\frac{1}{\sqrt{N}})$.
\end{thm}

\begin{proof}
    Let $\sigma^2 \Sigma$ be such that when accounting for the total of $S$ local steps, the composition of the sum-dominating mechanisms $\left(\gA^{(1)},\dots,\gA^{(S)} \right)$ is $(\varepsilon(\delta),\delta)$-DDP. 
    From Definition~\ref{def:infinite_divisibility} 
    it immediately follows that the required total noise can be broken into parts while staying in the same noise family to give 
    $$\sigma^2 = \sum_{i=1}^N \sigma_{LDP}^2 
    \Leftrightarrow \sigma_{LDP}^2 = \frac{\sigma^2}{N} 
    \Leftrightarrow \sigma_{LDP} = \frac{\sigma}{\sqrt{N}} .$$
\end{proof}

We note that similar scaling results as given in Theorem~\ref{thm:per-client-gaussian-noise-scaling} for specific cases like the Gaussian mechanism are already well-known in the existing literature (see, e.g., \citealt{Noble_Bellet_Dieuleveut_2023, Heikkila_2017}).

Finally, considering tightness of the privacy accounting done based on Theorem~\ref{thm:dpsgd_with_local_steps}, 
it is worth noting that 
since the accounting relies on Lemma~\ref{lemma:from_vector_to_exchanged_sums}, which assumes releasing each local step while the actually released query answer is a sum over the local steps, the resulting privacy bound need not be tight but an upper bound on the privacy budget. However, this matches the usual DP-SGD privacy accounting analysis (see e.g. \citealt{Mironov_Talwar_Zhang_2019,Koskela2020-ol}), which typically needs to account for each local optimization step due to technical reasons even if only the final model is released. In the general case, 
it has also been shown that hiding the intermediate steps does not bring any privacy benefits compared to the per-step accounting \citep{Annamalai_2024}.

%%%%%%%%%%%%%%%%%%%%%%%%%%%%%%%%%%%%%%%%
%%%%%%%%%%%%%%%%%%%%%%%%%%%%%%%%%%%%%%%%

\section{Experiments}
\label{sec:experiments}

\paragraph{Setup and Motivation:}

Our chosen settings try to mimic a typical cross-silo FL setup: there are a limited number of clients, each having a smallish local database. The clients have enough local compute to run optimization on the chosen model, while the number of server-client communications required for updating the global model are the main bottle-neck. Note that this bottleneck will emerge even with larger actual organisations training models with broadband connections, when the model size is large-enough, e.g., when training foundation models \citep{Bommasani2021FoundationModels}. This is especially true when using SecAgg protocols, since the cost of running a real SecAgg algorithm presents significant compute and communication overheads even with the efficient protocols discussed in Appendix~\ref{sec:SecAgg_in_practice}. 
In this setting, it makes sense to try and push more optimization steps to the clients while reducing the number of global updates (FL rounds). We also assume that the clients send their local updates via some trusted aggregator (which we only assume and do not implement in practice in the experiments. However, we do use only discrete DP mechanisms compatible with standard SecAgg algorithms in all the experiments). 
For more details on all the experiments, see Appendix~\ref{sec:experimental_details}. 
The code for reproducing the results will be released on GitHub.%
\footnote{\url{https://github.com/mixheikk/DPFL-with-SecAgg-paper}}

%%%%%%%%%%%%%%%%%%%%%%%%%%%%%%%%%%%%%%%%

\paragraph{CNN on Fashion-MNIST:}

We first train a small convolutional neural network (CNN) on Fashion MNIST data \citep{xiao2017_fashionMNIST}, that is distributed iid among 10 clients. We use the CNN architecture introduced by \citet{Papernot_Thakurta_Song_Chien_Erlingsson_2020, Tramer_Boneh_2020}. 
Figure~\ref{fig:fashion_mnist_skellam} shows the mean with standard error of the mean (SEM) over 5 repeats for test accuracy and loss with DP-SGD using Skellam noise \citep{Agarwal_Kairouz_Liu_2021} with 32 bits gradient quantization,.i.e., without quantization. We train the model for 20 FL rounds and varying number of local steps. Comparing the results for $1$ local step as opposed to $1$ local epoch ($\simeq 11$ steps, but with different sampling fraction compared to baseline), it is evident that being able to take more local optimization steps (as allowed by Theorem~\ref{thm:dpsgd_with_local_steps}) brings considerable utility benefits under fixed privacy and communication budgets.

\begin{figure}[htb]
\centering
\begin{subfigure}[b]{.5\linewidth}
\centering
\centerline{\includegraphics[width=\linewidth,trim={1cm 1cm 1cm 1cm},clip]{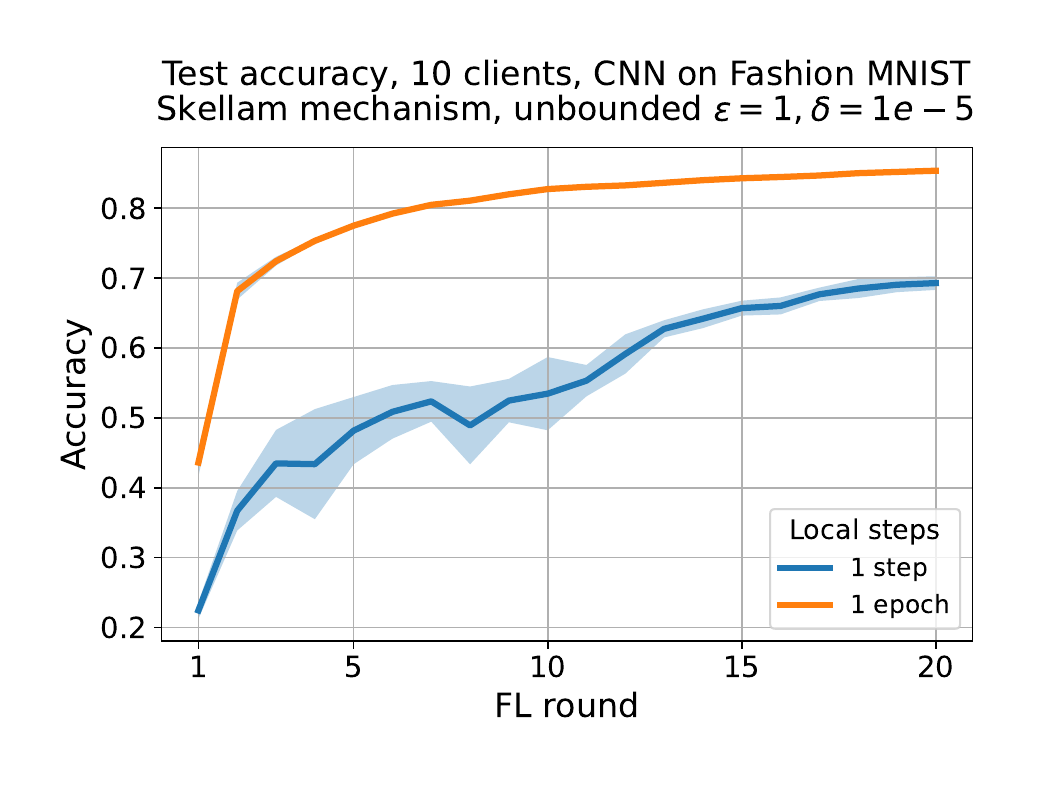}}
\caption{Test accuracy}
\label{fig:fashion_mnist_skellam_acc}
\end{subfigure}
\hfill
\begin{subfigure}[b]{.48\linewidth}
\centering
\centerline{\includegraphics[width=\linewidth,trim={1cm 1cm 1cm 1cm},clip]{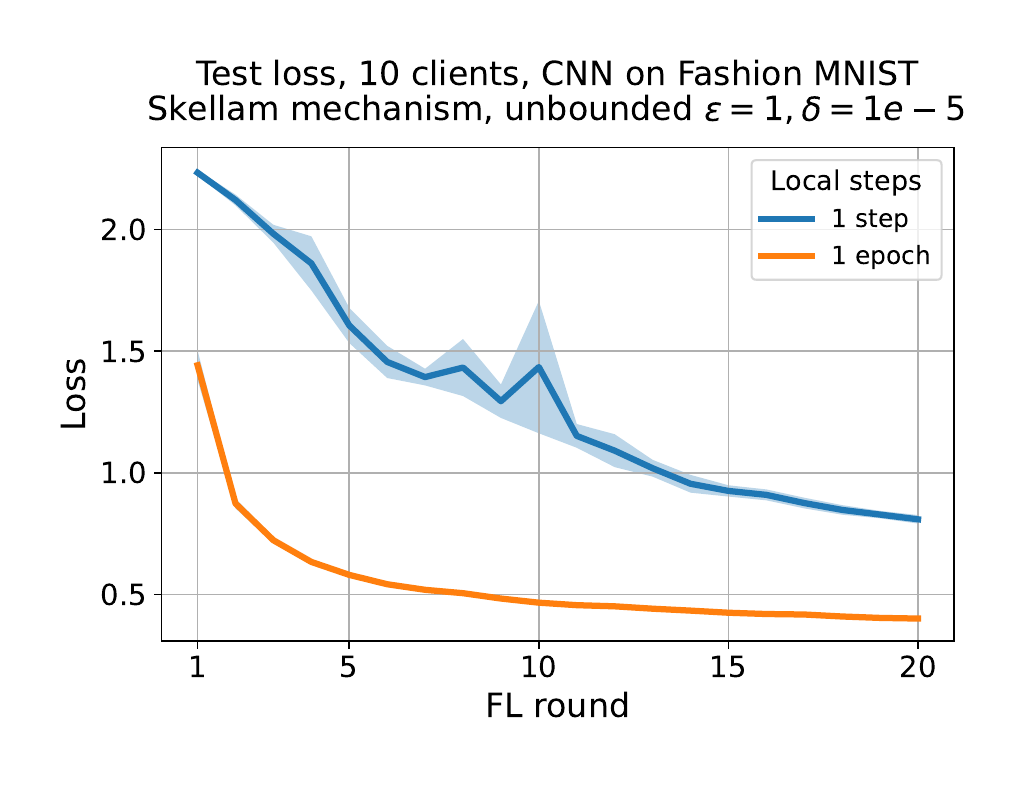}}
\caption{Test loss}
\label{fig:fashion_mnist_skellam_loss}
\end{subfigure}
\caption{CNN on Fashion-MNIST, 10 clients, mean and SEM over 5 seeds. Running more local steps is clearly beneficial.
\label{fig:fashion_mnist_skellam}}
\vspace{-1em}
\end{figure}

%%%%%%%%%%%%%%%%%%%%%%%%%%%%%%%%%%%%%%%%

\paragraph{Linear Classifier on Transformed CIFAR-10:}

Overall, assuming a fixed privacy budget, we might expect the benefits from being able to run more local steps to be more accentuated with more complex models and very limited communication budget, while for simple-enough models and more FL rounds even a few local steps could lead to good results. To test to what extent this is true for simple yet still useful models, we consider CIFAR-10 data \citep{Krizhevsky_2009}. Similar to \cite{Tramer_Boneh_2020}, we take a ResNeXt-29 model \citep{Xie_2017} pre-trained with CIFAR-100 data \citep{Krizhevsky_2009}, remove the final classifier, and use it as a feature extractor to transform the input data. We distribute the transformed CIFAR-10 data iid to 10 clients, and train linear classification layers from scratch for $10,20,40,80$ and $160$ FL rounds using DP-SGD with Skellam noise, 32 bit gradient quantization, and more local steps (1 epoch $\simeq 19$ steps, but with different expected batch size compared to the baseline). 
The mean and SEM over 5 seeds of the best results for each model over the training run are shown in Figure~\ref{fig:cifar10_pretrained_skellam}. 
The benefits of being able to run more than a single local steps are again clear; even with the relatively simple linear model, using $1$ local step needs roughly an order of magnitude more FL rounds over a fairly broad range of available communication budgets to reach a similar performance compared to using $1$ local epoch.

\begin{figure}[htb]
\centering
\begin{subfigure}[b]{.49\linewidth}
\centering
\centerline{\includegraphics[width=\linewidth,trim={1cm 1cm 1cm 1cm},clip]{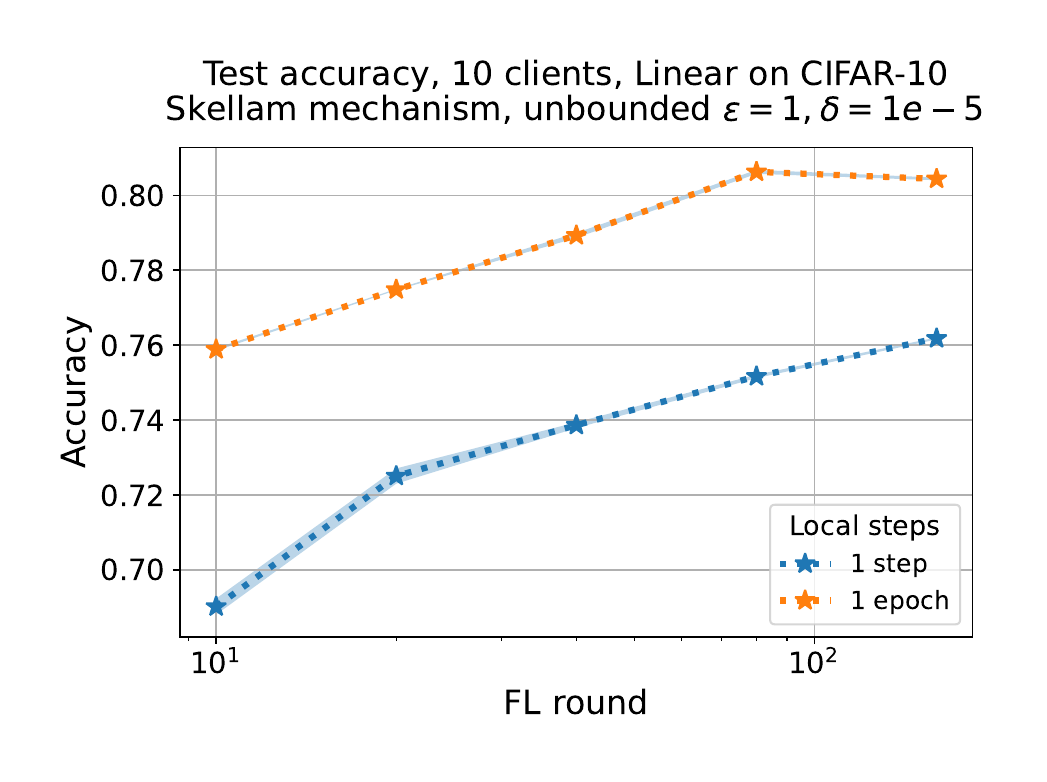}}
\caption{Test accuracy}
\label{fig:cifar10_pretrained_skellam_acc}
\end{subfigure}
\hfill
\begin{subfigure}[b]{.49\linewidth}
\centering
\centerline{\includegraphics[width=\linewidth,trim={1cm 1cm 1cm 1cm},clip]{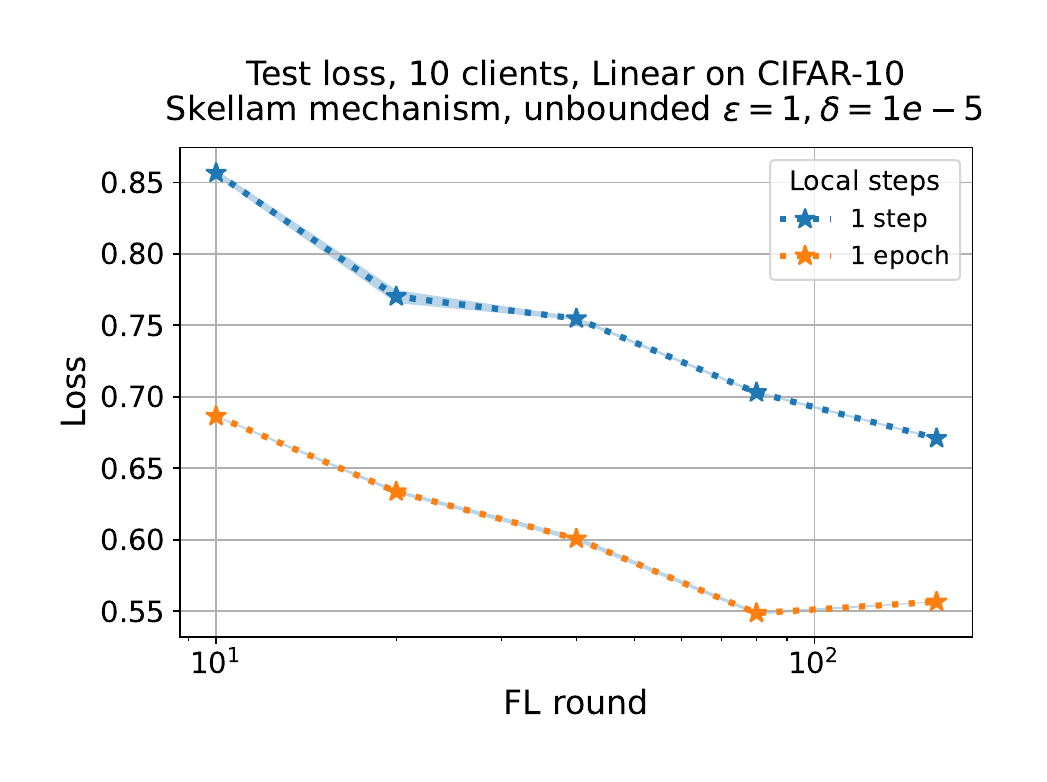}}
\caption{Test loss}
\label{fig:cifar10_pretrained_skellam_loss}
\end{subfigure}
\caption{Mean and SEM over 5 seeds of the best performance over training runs for Linear models on CIFAR-10 using pre-trained ResNeXt29 as feature extractor for varying number of FL rounds, 10 clients. Running more local steps is clearly beneficial.
\label{fig:cifar10_pretrained_skellam}}
\vspace{-1em}
\end{figure}

%%%%%%%%%%%%%%%%%%%%%%%%%%%%%%%%%%%%%%%%

\paragraph{Logistic Model on Income:}

To further test the robustness of the possible benefits from being able to run more than a single local optimization step, we train a simple Logistic Neural Network (LNN) model (i.e., 1-layer fully connected linear classification network similar to the one used in the previous experiment, but without any pre-trained feature extractor) on ACS Income data \citep{DBLP:conf/nips/DingHMS21}. Unlike the synthetic iid data splits used in the previous experiments, Income data has an inherent client split corresponding to $51$ states from where the data has been collected. Since the inherent split is heterogeneous (different states have very different number of samples as well as different data distributions), we would expect the benefits of doing more local optimization steps between global communication rounds to dwindle, since the local models from different clients could diverge when only trained locally. However, as shown in Figure~\ref{fig:income_skellam}, even in this setting taking more local steps can be very beneficial (here $1$ epoch $\simeq 20$ steps with same local sampling fraction compared to baseline). This clearly demonstrates the utility of our analysis.

\begin{figure}[htb]
\centering
\begin{subfigure}[b]{.51\linewidth}
\centering
\centerline{\includegraphics[width=\linewidth,trim={1cm 1cm 1cm 1cm},clip]{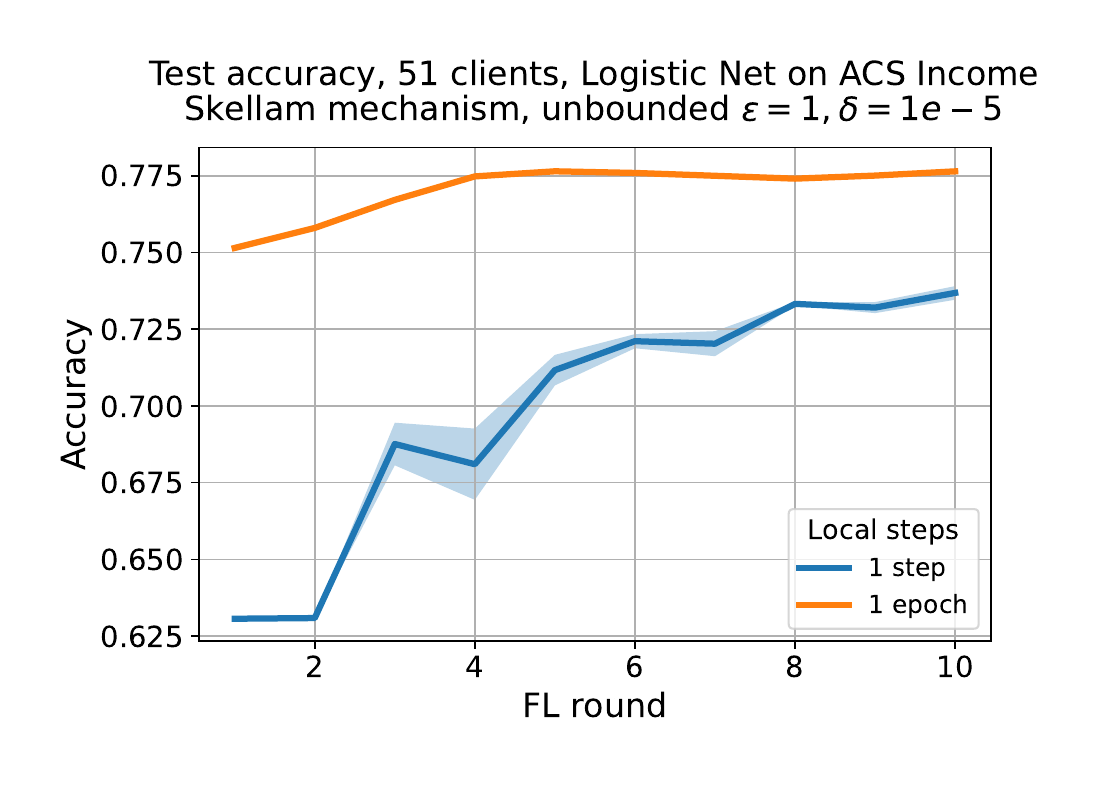}}
\caption{Test accuracy}
\label{fig:income_skellam_acc}
\end{subfigure}
\hfill
\begin{subfigure}[b]{.48\linewidth}
\centering
\centerline{\includegraphics[width=\linewidth,trim={1cm 1cm 1cm 1cm},clip]{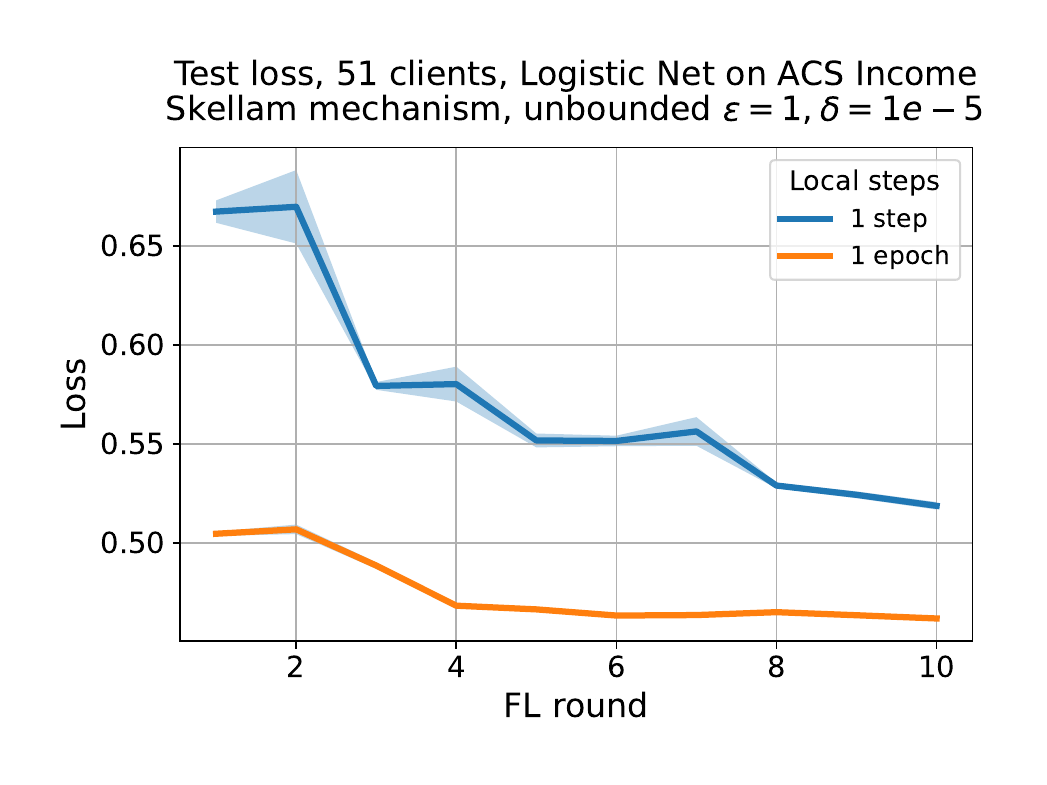}}
\caption{Test loss}
\label{fig:income_skellam_loss}
\end{subfigure}
\caption{LNN on ACS Income, 51 clients, mean and SEM over 5 seeds. Running more local steps is clearly beneficial.
\label{fig:income_skellam}}
\vspace{-1em}
\end{figure}

%%%%%%%%%%%%%%%%%%%%%%%%%%%%%%%%%%%%%%%%

\paragraph{Improved Privacy from Model Averaging:}

Finally, our results might sometimes be of use also outside standard FL. For example, consider a setting where we have $N$ copies of a model trained on disjoint data sets (for example, think of independent parties learning a classifier on top of a common pre-trained model or fine-tuning a common pre-trained model; either scenario would usually lead to shared model structure and possibly also hyperparameters without explicit coordination), and the parties would like to combine the models post-hoc without running any joint training from scratch. Since this can be seen as FL with a single FL round, if the original model training on each party satisfies Assumption~\ref{assumption:common_assumptions} (or the relaxed assumptions in Appendix~\ref{sec:loosening_assumptions}), then a simple averaging of the weights will result in a joint model with improved privacy guarantees against adversaries without access to the original models.

To demonstrate this effect, we account for privacy assuming the same linear model used in Figure~\ref{fig:cifar10_pretrained_skellam} (but without actually training any models), Skellam mechanism with $32$ bit gradients,  Poisson subsampling with sampling probability $0.1$, and varying number of parties and local steps. 
The accounting is done as it would be done in a realistic setting: we first find a noise level $\sigma_{LDP}$ that results in the target privacy level (unbounded $(\varepsilon=5,\delta=1e-5)$-LDP) for each separate model with the chosen number of local steps. We then assume that the local training satisfies Assumption~\ref{assumption:common_assumptions} and calculate the privacy for averaging varying number of local models. Combining even $2$ models results in clearly improved privacy for the averaged model (see Table~\ref{table:1-FL-round-privacy} in Appendix~\ref{sec:additional_results}).

%%%%%%%%%%%%%%%%%%%%%%%%%%%%%%%%%%%%%%%%
%%%%%%%%%%%%%%%%%%%%%%%%%%%%%%%%%%%%%%%%

\section{Discussion}
\label{sec:discussion}

In this paper we have shown how to combine multiple local steps in DPFL using fine-grained protection granularities with SecAgg, and empirically demonstrated that this can bring considerable utility benefits under various communication-constrained settings. 
Our experimental results stand in stark contrasts with the message from the currently existing theoretical bounds for vanilla DPFL with FedAvg \citep[Theorem 3.2]{Malekmohammadi_Yu_Cao_2024}, which do not show any benefit from increasing the number of local steps. 
This disagreement of experimental and theoretical results underlines the need for improved theoretical analysis to understand the conditions under which increasing the number of local steps can lead to improved utility, similar to the recent breakthroughs in analysing non-DP FL \citep{pmlr-v162-mishchenko22b}.

%%%%%%%%%%%%%%%%%%%%%%%%%%%%%%%%%%%%%%%%
%%%%%%%%%%%%%%%%%%%%%%%%%%%%%%%%%%%%%%%%
%%%%%%%%%%%%%%%%%%%%%%%%%%%%%%%%%%%%%%%%

%-------------------------------------------------------------------------------
\subsubsection*{Acknowledgments}
%-------------------------------------------------------------------------------

This research was partially supported by: 1) The Ministry of Economic Affairs and Digital Transformation of Spain and the European Union-NextGenerationEU programme for the "Recovery, Transformation and Resilience Plan" and the "Recovery and Resilience Mechanism" under agreements TSI-063000-2021-142, TSI-063000-2021-147 (6G-RIEMANN) and TSI-063000-2021-63 (MAP-6G), 2) The European Union Horizon 2020 program under grant agreements No.101096435 (CONFIDENTIAL6G) and No.101139067 (ELASTIC), and 3) the European Lighthouse on Safe and Secure AI (ELSA) from the European Union’s Horizon Europe program under grant agreement No 101070617. The views and opinions expressed are those of the authors only and do not necessarily reflect those of the European Union or other funding agencies. Neither the European Union nor the granting authority can be held responsible for them.

%-------------------------------------------------------------------------------

\bibliography{main}
\bibliographystyle{tmlr}

\appendix
\section{Appendix}
\label{sec:appendix}

%%%%%%%%%%%%%%%%%%%%%%%%%%%%%%%%%%%%%%%%

\subsection{Definitions Omitted from the Main Paper}
\label{sec:additional_defs}

In the following, we write $\stackrel{d}{=}$ to mean equality of distributions. With this notation, we define infinite divisibility in a standard way (see, e.g., \citealt{feller-intro-prob}):
\begin{defn}[Infinite divisibility]
\label{def:infinite_divisibility}
    A distribution $F$ is infinitely divisible if for every $n$ there exists a family of distributions $F_n$ s.t. 
    $$ S_n \stackrel{d}{=} \sum_{i=1}^n X_{i,n}, $$
    where the random variables $S_n \sim F$ and $X_{i,n} \sim F_n$ are independent for all $i=1,\dots,n$.
\end{defn}

%%%%%%%%%%%%%%%%%%%%%%%%%%%%%%%%%%%%%%%%
\subsection{Standard Algorithms Omitted from the Main Paper}
\label{sec:standard_algos}

For convenience, we give the pseudo-code for running DPFL with FedAvg when using DP-SGD for either sample- (Algorithm~\ref{alg:DPFL-sample-level-client-side}) or user-level (Algorithm~\ref{alg:DPFL-user-level-client-side}) protection. We emphasize that we do not claim any novel contribution in these algorithms, but they are included simply to clarify the discussion in the main paper.

\begin{algorithm}[H]
	\caption{(DP)FL with FedAvg: Server-side \citep{McMahan_2016}}
    \label{alg:DPFL-server-side}
	\begin{algorithmic}[1]
        \Require Number of clients $N$, number of FL rounds T, initial global model parameters $\theta^{(0)}$
        \For{$t \in [T]$}
            \State Choose $N_t$ clients for updating (w.l.o.g. numbered $1,\dots,N_t$) and send the current model $\theta^{(t-1)}$
            \State Receive pseudo-gradients $\Delta_i^{(t-1)}, i=1,\dots,N_t$ from the updating clients
            \State Update the global model: 
            \State $\theta^{(t)} \leftarrow \theta^{(t-1)} +\frac{1}{N_t} \sum_{i=1}^{N_t} \Delta_i^{(t-1)}$
        \EndFor
	\end{algorithmic}
	\textbf{Output: $\theta^{T}$}\\
\end{algorithm}

\begin{algorithm}[H]
	\caption{DPFL: Client-side with sample-level DP-SGD using Gaussian noise \citep{Song_Chaudhuri_Sarwate_2013,DBLP:conf/ccs/AbadiCGMMT016}}
    \label{alg:DPFL-sample-level-client-side}
	\begin{algorithmic}[1]
        \Require Number of local optimization steps $S$, local dataset $x_i$, loss function $f$, learning rate $\gamma$, clipping constant $C$, noise scale $\sigma$, Poisson subsampling probability $q$
        \State \textbf{Client receives current model $\theta$ from the server:}
        \State $\theta_0 \leftarrow \theta$
        \For{$s \in [S]$}
            \State \textbf{Update local model with DP-SGD:}
            \State $g_{i,s} \leftarrow 0$
            \State Sample minibatch $b_s$ from $x_i$ using Poisson sampling with probability $q$
            \For{$j \in [|b_s|]$}
                \State \textbf{Calculate $j$th per-sample gradient in minibatch, clip and accumulate:}
                \State $\tilde g_{i,s} \leftarrow \nabla f(b_{s,j}; \theta_{s-1})$
                \State $g_{i,s} \leftarrow g_{i,s} + \min\left(1, \frac{C}{\|\tilde g_{i,s} \|_2} \right) \tilde g_{i,s}$ 
            \EndFor
            \State \textbf{Add noise to clipped gradients and take a step:}
            \State $\theta^{(s)} \leftarrow \theta^{(s-1)} - \gamma \left( g_{i,s} + C \xi \right)$, where $\xi \sim \gN (0, \sigma^2 I_d)$ and $d$ is model dimensionality
            \label{alg:DPFL-line:sample-level-DPSGD-step}
        \EndFor
        \State $\Delta \leftarrow \theta^{S} - \theta_0$
	\end{algorithmic}
	\textbf{Output: $\Delta$} (Pseudo-gradients sent to the server)\\
\end{algorithm}

\begin{algorithm}[H]
	\caption{DPFL: Client-side with user-level DP using Gaussian noise \citep{DBLP:conf/iclr/McMahanRT018}}
    \label{alg:DPFL-user-level-client-side}
	\begin{algorithmic}[1]
        \Require Number of local optimization steps $S$, local dataset $x_i$, learning rate $\gamma$, clipping constant $C$, noise scale $\sigma$, Poisson subsampling probability $q$
        \State \textbf{Client receives current model $\theta$ from the server:}
        \State $\theta_0 \leftarrow \theta$
        \State \textbf{Update local model $S$ times with standard SGD:}
        \State $\theta^{(S)} \leftarrow$ from running SGD locally for $S$ steps with learning rate $\gamma$ on $x_i$
        \State $\Delta \leftarrow \theta^{S} - \theta_0$
        \State \textbf{Clip pseudo-gradients and add noise:}
        \State $\Delta \leftarrow \min \left(1,\frac{C}{\| \Delta \|_2} \right) \Delta + C\xi $, where $\xi \sim \gN \left(0,  \sigma^2 I_d \right)$ and $d$ is model dimensionality
	\end{algorithmic}
	\textbf{Output: $\Delta$} (Pseudo-gradients sent to the server)
\end{algorithm}

Essentially, for this example when aiming for DDP, at FL round $t$ our analysis enables choosing the number of local steps $S$ in Algorithm~\ref{alg:DPFL-sample-level-client-side} freely while still calibrating the DP noise by looking at the total noise level $\sum_{i=1}^{N_t} \sigma_i^2$ added by each client separately on each local step (line \ref{alg:DPFL-line:sample-level-DPSGD-step} in Algorithm~\ref{alg:DPFL-sample-level-client-side}). Note that for the special case of using Gaussian noise, a similar result has been presented by \cite{Noble_Bellet_Dieuleveut_2023}.

%%%%%%%%%%%%%%%%%%%%%%%%%%%%%%%%%%%%%%%%

\subsection{Loosening Assumptions}
\label{sec:loosening_assumptions}

In the main paper, as stated in Assumption~\ref{assumption:common_assumptions}, for each FL round we have assumed constant learning rate $\gamma$, norm clipping bound $C$, noise level $\sigma$, and number of local optimization steps $S$, all of them shared by all the clients. Next, we consider loosening these assumptions. As before, w.l.o.g. we consider only a single FL round, and will therefore omit the index $t$.

Focusing first on the learning rate, we can immediately generalize our results to allow for different learning rate $\gamma_s$ for each local step $s=1,\dots,S$:
with this notation, following the reasoning of Theorem~\ref{thm:dpsgd_with_local_steps}, the aggregated update from the clients is given by 
\begin{equation}
    \sum_{i=1}^N \Delta_i = -\sum_{i=1}^N \sum_{s=1}^S \gamma_s \gA_{i}^{(s)} (\data_i; z_{i,s}) ,
\end{equation}
which can again be seen as post-processing the vector $\left( \sum_{i=1}^N \gA_i^{(1)}(\data_i), \dots, \sum_{i=1}^N \gA_i^{(S)}(\data_i) \right)$, so we can again use Lemma~\ref{lemma:from_vector_to_exchanged_sums} for accounting without encountering problems.

When considering client-specific learning rates things can be more complicated. The main issue now is to find proper sum-dominating mechanisms that satisfy: 
\begin{equation}
    \sup_{\data \simeq \data'} H_{\alpha} \left( \sum_{i=1}^N \gamma_{i,s} \gA_i^{(s)} (\data_i) \| \sum_{i=1}^N \gamma_{i,s} \gA_i^{(s)} (\data'_i) \right) \leq 
    \sup_{\data \simeq \data'} H_{\alpha} \left( \gA^{(s)} (\data) \| \gA^{(s)} (\data') \right), s=1,\dots,S .
\end{equation}

As a concrete example, assume $\gA_i^{(s)}$ is the continuous Gaussian mechanism with shared norm clipping constants and noise levels $C_{i,s} = C_s, \sigma_{i,s} = \sigma_s \ \forall i$. Dropping the step index $s$ for readability, let $\gamma_i = \frac{\gamma_1}{l_i}$ for some $l_i > 0, i=2,\dots,N$. Writing $g_i$ for a sum over the per-unit clipped gradients of client $i$, and $\xi_i \sim \mathcal N(0, C^2 \sigma^2 \cdot I_d)$ a single optimization step now contributes the following term for the global update: 
\begin{align}
    & -\sum_{i=1}^N \gamma_i \left( g_{i} + \xi_i \right) \\
    = & - \gamma_1 \left( g_1 + \xi_1 + \sum_{i=2}^N \frac{g_{i} + \xi_i }{l_i} \right) \\
    \label{eq:single_step_general_form}
    = & - \gamma_1 \left( g_1 + \sum_{i=2}^N \frac{g_i}{l_i} + \xi \right), 
\end{align}
where 
$\xi \sim \mathcal N (0, C^2 \sigma^2 [1 + \sum_{i=2}^N \frac{1}{l_i^2} ] \cdot I_d)$, which is a sum-dominating Gaussian mechanism. When accounting for the sum-dominating mechanism, it has sensitivity 
$C^* = \max \{C, \frac{C}{l_2}, \dots, \frac{C}{l_N} \}$, which in turn gives noise variance 
$(\frac{C}{C^*})^2 \sigma^2 [1 + \sum_{i=2}^N \frac{1}{l_i^2} ]$ for DP.

Similarly, we could relax the assumptions further to allow the clients to use different clipping and noise levels $C_i,\sigma_i$. As before, a single optimization step can again be written in the form of \Eqref{eq:single_step_general_form}, when 
\begin{equation}
    \xi \sim \mathcal N \left(0, \left[ C_1^2 \sigma_1^2 + \sum_{i=2}^N \frac{C_i^2 \sigma_i^2}{l_i^2}  \right] \cdot I_d \right). 
\end{equation}
For global privacy accounting with a sum-dominating Gaussian mechanism, suitable sensitivity is now given by 
$C^* = \max\{ C_1, \frac{C_2}{l_2},\dots, \frac{C_N}{l_N} \}$, 
and the resulting variance for accounting is 
$\sum_{i=1}^N (\frac{C_i \sigma_i }{C^* })^2$.

Assuming clients have differing number of local steps, we can try to fuse some local steps for the privacy analysis until all clients have the same number of steps $S$, after which we can then use the earlier results.%
\footnote{Alternatively, we could also consider breaking some local steps into several parts. We leave the detailed consideration of this approach for future work.}

As a simple example, assume we have 2 clients running DP-SGD: 
client 1 runs $S$ local steps using norm clipping constant $C$ and Gaussian mechanism with noise variance $\sigma^2$, while client 2 runs $2S$ local steps with clipping $C/2$ and Gaussian noise variance $\sigma^2$. The difference now is that while the clipping is done on each step, from the privacy accounting perspective we can disregard some noise and think that client 2 adds noise only on every other step. 
Looking at the update from client 2, we would then have 

\begin{align}
    \Delta_2 &= -\gamma \sum_{s=1}^{2S} (g_{2,s} + \mathbb I[s = 2l, l\in \mathbb N] \cdot \xi_{2,s}) \\
    \label{eq:fused-local-steps}
    &= -\gamma \sum_{s=1}^{S} (g'_{2,s} + \xi'_{2, s}),
\end{align}
where $g_{2,s}$ are the clipped per-sample gradients, $g'_{2,s} := g_{2,2s-1} + g_{2,2s}$, 
$\xi_{2,s}$ are the noise values, 
$\xi'_{2,s} := \xi_{2,2s}$, and $\mathbb I$ is the indicator function. Due to the clipping, the sensitivity of each fused step can be easily upper bounded via triangle-inequality: 
$\| g_{2,s'} \|_2 = \| g_{2,2s'-1} + g_{2,2s'} \|_2 \leq \| g_{2,2s'-1}\|_2 + \| g_{2,2s'} \|_2 \leq C$. Since \Eqref{eq:fused-local-steps} now has the same number of local steps as client 1 is taking, we can readily use the previous results to enable privacy accounting for the aggregated update. Combining the fusing of local steps with the previous notes on differing clipping norm values, learning rates and noise variances allows us to use our main results in several settings beyond what is stated in Assumption~\ref{assumption:common_assumptions}.

As a final note, when the clients use data subsampling for the local optimization, differing local subsampling probabilities can lead to having varying DP guarantees between the clients on the global level due to the different subsampling amplification effects, but can otherwise be incorporated with the same analysis we have already presented.

%%%%%%%%%%%%%%%%%%%%%%%%%%%%%%%%%%%%%%%%

\subsection{From Ideal Trusted Aggregators to Practical SecAgg Protocols}
\label{sec:SecAgg_in_practice}

For implementing the trusted aggregator assumed in Theorem~\ref{thm:dpsgd_with_local_steps} in practice, it should be noted that as the sum over $s$ is done locally by each client during local optimization, it is always trusted as long as the individual clients are, while the sum over $i$ would need to be implemented, e.g., using a suitable SecAgg protocol. 
Several such algorithms are known, including the ones proposed by \cite{BellBGL020, BonawitzIKMMPRS17, SabaterBR22, SoGA21a}.

Using a SecAgg protocol will typically also place some extra requirements on the DP mechanisms $\gA_{i}^{(s)}$, since the SecAgg algorithms usually run on elements of finite rings. This precludes continuous noise mechanisms. 
A viable alternative is to use some suitable discrete noise mechanism, such as Skellam \citep{Agarwal_Kairouz_Liu_2021} or Poisson-binomial \citep{Chen_et_al_2022}. 
However, differing from the cases considered in the cited papers, 
since in our case the clients send model updates instead of single gradients, the finite ring size used in the SecAgg protocol needs to accommodate the model update size: it does no good to use Skellam mechanism with gradient quantization to a small number of bits, if the model weights and the resulting model update $\Delta_i$ for client $i$ still uses 32 bit floats.

%%%%%%%%%%%%%%%%%%%%%%%%%%%%%%%%%%%%%%%%

\subsection{Privacy Accounting Details}
\label{sec:accounting_details}

For privacy accounting we utilize R\'enyi DP (RDP):
\begin{defn}
    \label{def:rdp}
    \citep{Mironov_2017}
	Let $\alpha > 1$ and $\varepsilon > 0$. 
	A randomised algorithm $\gA : \, \dataSpace^* \rightarrow \mathcal{O}$ is  $(\alpha,\veps)$-RDP 
	if for every $\data, \data' \in \dataSpace^*: \data \simeq \data'$ 
	\begin{equation*}
		\begin{aligned}
			& D_{\alpha} ( \gA (\data) \| \gA (\data') )  \leq \varepsilon ,
		\end{aligned}
	\end{equation*}
 where $D_{\alpha}$ is the R\'enyi divergence of order $\alpha$:
 \begin{equation*}
     D_{\alpha} (P \| Q) = \frac{1}{\alpha-1} \log \mathbb E_{t \sim Q} \left( \frac{p(t)}{q(t)} \right)^{\alpha}.
 \end{equation*}
\end{defn}

We do privacy accounting for all the experiments based on RDP. Generally, we account for the privacy of each individual local optimization step using the noise contributions from all the clients selected for a given FL round. When the clients use Poisson subsampling to sample minibatches (we assume each client uses the same probability for including any individual sample in the minibatch), we use standard RDP privacy amplification results. In practice, we use the RDP accountant implemented in Opacus \citep{opacus}, as well as bounds for Skellam mechanism by \citet{Agarwal_Kairouz_Liu_2021} and tight RDP amplification by Poisson subsampling \citep{Steinke_2022}. We calculate the privacy cost of the entire training run in RDP, and then convert into ADP using \citep[Proposition 3]{Mironov_2017}. Note that, as is common in DP research, we do not include the privacy cost of hyperparameter tuning in the reported privacy budgets (see, e.g., \citealt{Tramer_Boneh_2020} for some reasoning on this practice).

%%%%%%%%%%%%%%%%%%%%%%%%%%%%%%%%%%%%%%%%

\subsection{Experimental Details}
\label{sec:experimental_details}

All the experimental settings we use satisfy Assumption~\ref{assumption:common_assumptions}. 
We use DP-SGD with Skellam mechanism to optimise the local model parameters, and standard federated averaging as the aggregation rule for updating the global model in all experiments. For each centralised data set (combining original train and test sets), we split the data randomly into equal shares, which results in having almost the same data distribution on each client. 
For hyperparameter optimization with each dataset, we first split each clients' data internally into train and test parts with fractions (.8-.2). For tuning all hyperparameters, we use only the training fraction, and divide it further (.7-.3) into hyperparameter train-validation. We use Bayesian optimization-based approach implemented in Weights and Biases \citep{wandb} for hyperparameter tuning, and simulate FL using Flower \citep{beutel2020flower}.

In general, when tuning hyperparameters we do 50 hyperparameter tuning runs. For each tuning run, we train the model on hyperparameter training fraction, test on the validation fraction, and try to optimise for the final model weighted validation loss. After finishing the hyperparameter tuning, we re-train the model from scratch 5 times with different random seeds with the best found hyperparameters using the entire original training data and testing on the test fraction. 
We report the mean and the standard error of the mean (SEM) in all the figures. In Figure~\ref{fig:cifar10_pretrained_skellam} we plot the minimum test loss/maximum test accuracy taken over the entire training run.

For the experiments with Fashion-MNIST \citep{xiao2017_fashionMNIST} and CIFAR-10 \citep{Krizhevsky_2009} data sets, we run hyperparameter tuning separately for each combination of number of local steps \{1 step, 1 epoch\}, and expected minibatch sizes on the grid $\{64, 128,256,512\}$ using Poisson subsampling.

For Fashion-MNIST the best expected batch sizes found are $512$ for 1 local epoch, and $128$ for 1 local step.

With CIFAR-10, due to heavy computational cost of hyperparameter tuning, we use a single expected batch size for each configuration of local steps \{1 step, 1 epoch\} and FL rounds $\{10,20,40,160\}$. Concretely, we pick the best expected batch size value from the above grid when using Bayesian optimization to tune all hyperparameters with $20$ FL rounds. This results in choosing expected batch size $128$ for $1$ local step and $256$ for $1$ local epoch. We then use these values and optimize all other hyperparameters separately for all other FL round settings.

With ACS Income data \citep{DBLP:conf/nips/DingHMS21}, we tune all hyperparameters for each combination of local steps \{1 step, 1 epoch\} with Poisson subsampling using local sampling probability on the grid $\{0.4, 0.2, 0.1, 0.05\}$ for $10$ FL rounds. We report results on the best found local sampling probabilities ($0.05$ for both).

For ResNeXt-29 8x64, we used pre-trained weights available from \url{https://github.com/bearpaw/pytorch-classification}. Our implementation of the Skellam mechanism is based on the implementation from \url{https://github.com/facebookresearch/dp_compression} \citep{Chaudhuri_Guo_Rabbat_2022,Guo_Chaudhuri_Stock_Rabbat_2023}.

For American Community Survey (ACS) Income data set \cite{DBLP:conf/nips/DingHMS21} we use the data for all the states and Puerto Rico for 2018. The goal is to predict whether an individual has income greater than \$$50000$. Instead of simulating data splits, we use the inherent splits, i.e., we take each original region (state or Puerto Rico) to be a client.

For training all models, we use a small cluster with NVIDIA Titan Xp, and NVIDIA Titan V GPUs. The total compute time of all the training runs (including debugging) over all GPUs amounts roughly to 30-60 GPU days.

\subsection{Additional Results}
\label{sec:additional_results}

Table~\ref{table:1-FL-round-privacy} shows the results from averaging several independently trained LDP models as described in Section~\ref{sec:experiments}.

\begin{table}[htb]
\caption{Improved privacy for averaged models, Skellam mechanism, 32 bits (no quantization), Poisson sampling with sampling fraction $0.1$, each local model is unbounded $(\varepsilon=5.,\delta=1e-5$)-LDP. Averaging more models improves on the DP guarantees against adversaries who do not have access to the original models.
\label{table:1-FL-round-privacy}}
\begin{center}
\begin{tabular}{|| l | c | c | c ||} 
 \hline
 Local steps & Parties & $\sigma_{total}$ & avg model $\varepsilon$ \\ [0.5ex] 
 \hline\hline
 $1$ step & $1$ & $0.69$ & $5.0$ \\ 
 $1$ step & $2$ & $0.98$ & $2.78$ \\ 
 $1$ step & $5$ & $1.54$ & $1.22$ \\ 
 $1$ step & $10$ & $2.18$ & $0.64$ \\ 
 \hline
 $1$ epoch & $1$ & $0.90$ & $5.0$ \\ 
 $1$ epoch & $2$ & $1.28$ & $2.61$ \\
 $1$ epoch & $5$ & $2.02$ & $1.19$ \\ 
 $1$ epoch & $10$ & $2.85$ & $0.72$ \\ 
 \hline
 $5$ epochs & $1$ & $1.18$ & $5.0$ \\ 
 $5$ epochs & $2$ & $1.67$ & $2.85$ \\
 $5$ epochs & $5$ & $2.64$ & $1.55$ \\
 $5$ epochs & $10$ & $3.73$ & $1.03$ \\
 \hline
\end{tabular}
\end{center}
\end{table}

\end{document}